\documentclass{article}

\usepackage{arxiv}

\usepackage[utf8]{inputenc} 
\usepackage[T1]{fontenc}    
\usepackage{hyperref}       
\usepackage{url}            
\usepackage{booktabs}       
\usepackage{amsfonts}       
\usepackage{nicefrac}       
\usepackage{microtype}      
\usepackage{lipsum}		
\usepackage{graphicx}
\usepackage{natbib}
\usepackage{doi}
\usepackage[linesnumbered,ruled,vlined]{algorithm2e}
\usepackage{algorithm2e}
\usepackage{graphicx}
\usepackage{amsmath}
\usepackage{amsfonts}
\usepackage{mathtools}
\usepackage{amsmath, amsthm, amscd, amsfonts, amssymb, graphicx, color}

\usepackage{comment}

\newtheorem{theorem}{Theorem}[section]
\newtheorem{corollary}{Corollary}[section]
\newtheorem{definition}{Definition}[section]

\newtheorem{lemma}{Lemma}[section]

\title{New logarithmic step size for stochastic gradient descent}

\author{ M. Soheil Shamaee \\
	Department of  Computer Science,\\
   Faculty of Mathematical Science, \\
   {University of Kashan}, Kashan, Iran\\
{soheilshamaee@kashanu.ac.ir}\\
\And
S. Fathi Hafshejani\\
Department of Applied Mathematics\\
	Shiraz University of Technology\\
	 Shiraz, Iran\\
{s.fathi@sutech.ac.ir}\\
\And Z. Saeidian\\
Department of Mathematical Sciences\\
University of Kashan\\ Kashan, Iran\\
{saeidian@kashanu.ac.ir}
}

\renewcommand{\shorttitle}{New logarithmic step size}

\begin{document}
\maketitle

\begin{abstract}
In this paper, we propose a novel warm restart technique using a new logarithmic step size for the stochastic gradient descent (SGD) approach. For smooth and non-convex functions, we establish an $O(\frac{1}{\sqrt{T}})$ convergence rate for the SGD. We conduct a comprehensive implementation to demonstrate the efficiency of the newly proposed step size on the ~FashionMinst,~ CIFAR10, and CIFAR100 datasets. Moreover, we compare our results with nine other existing approaches and demonstrate that the new logarithmic step size improves test accuracy by $0.9\%$ for the CIFAR100 dataset when we utilize a convolutional neural network (CNN) model.
\end{abstract}


\section{Introduction}
Stochastic gradient descent (SGD), which dates back to the work by Robbins and Monro \cite{robbins1951stochastic} is widely observed in training modern Deep Neural Networks (DNNs), which are widely used to achieve state-of-the-art results in multiple problem domains like image classification problems \cite{krizhevsky2017imagenet,krizhevsky2009learning}, object detection \cite{redmon2017yolo9000}, and classification automatic machine translation \cite{zhang2015deep}.

The value of the step size (or learning rate) is crucial for the convergence rate of SGD. Selecting an appropriate step size value in each iteration ensures that SGD iterations converge to an optimal solution. If the step size value is too large, it may prevent SGD iterations from reaching the optimal point. Conversely, excessively small step size values can lead to slow convergence or mistakenly identify a local minimum as the optimal solution \cite{mishra2019polynomial}. To address these challenges, various schemes have been proposed. One popular approach is the Armijo line search method, initially introduced for SGD by Vaswani et al. \cite{vaswani2019painless}, which provides theoretical results for strong-convex, convex, and non-convex objective functions. Gower et al. \cite{gower2019sgd} proposed another approach that combines a constant step size with a decreasing step size. Their algorithm starts with a constant step size and switches to a decreasing step size after $4k$ iterations, where $k$ denotes the problem's condition number. This strategy guarantees convergence for strongly convex functions but requires knowledge of the condition number and is not applicable to non-convex functions. Additionally, the decay step size is commonly used in practice for solving non-convex problems encountered during DNN training \cite{krizhevsky2017imagenet, huang2017densely}.


Smith in \cite{smith2017cyclical} proposed an efficient method for setting the step size known as the  cyclical learning rate. {
Training neural networks with cyclical learning rates instead of fixed learning rates can lead to significant improvements in accuracy without the need for manual tuning, often requiring fewer iterations.} Loshchilov and Hutter \cite{loshchilov2016sgdr} utilized this strategy and proposed a {\em warm restart} technique for SGD that does not need to apply gradient information for updating step size in each iteration. {
The key idea behind warm restarts is that in each restart, the learning rate is initialized to a specific value (denoted as $\eta_0$) and scheduled to decrease.} Moreover, it has been shown that a warm restarted SGD takes 2– 4 times less time than a traditional learning rate adjustment strategy \cite{vrbanvcivc2022efficient}. In recent years, numerous step sizes with warm restarts have been introduced \cite{mishra2019polynomial,xu2020stochastic}. Vrbančič extended the idea of the cosine step size and proposed three different step sizes with warm restarts \cite{vrbanvcivc2022efficient}.

The convergence rate, from a theoretical standpoint, is a crucial metric for assessing the effectiveness and efficiency of an algorithm. In the case of the SGD algorithm, the convergence rate depends on various factors, such as the type of objective function (e.g., strong convex, convex, or non-convex) and the value of the step size used in each iteration. 
{For convex and smooth functions with Lipschitz continuous gradients, it has been proven that the SGD algorithm can achieve a convergence rate of $O\left(\frac{1}{\sqrt{T}}\right)$ \cite{nemirovski2009robust,ghadimi2013stochastic}.} The rate of convergence for a strongly convex function is $O\left(\frac{1}{T}\right)$ \cite{moulines2011non,rakhlin2011making}. When the goal is smooth and non-convex functions, Ghadimi and Lan \cite{ghadimi2013stochastic} provided an $O\left(\frac{1}{\sqrt{T}}\right)$ rate of convergence for SGD by a constant step size, i.e., $\eta_t=\frac{1}{\sqrt{T}}$. { In another work, 
Li et al. \cite{li2021second} established  an $O\left(\frac{1}{\sqrt{T}}\right)$ rate of convergence for both cosine and exponential step sizes which were first considered in \cite{loshchilov2016sgdr}. Moreover, their empirical results confirm that both step sizes have the best accuracy and training loss achievements.  Ge et al. \cite{ge2019step} considered  decay step size for solving the least squares problems. Wang et al. \cite{wang2021convergence}  took into account the influence of the probability distribution, denoted as $\frac{\eta_t}{\sum_{t=1}^T\eta_t}$, on both the implementation and theoretical aspects of the algorithm. They demonstrated that within this framework, the SGD algorithm with the exponential step size achieves a convergence rate of $O\left(\frac{\ln T}{\sqrt{T}}\right)$. Notably, their findings reveal that assigning a higher probability to the final iterations through the step size leads to enhancements in accuracy and improvements in the loss function.}


{
In the context of many decay step sizes, the convergence rate of the SGD algorithm is commonly analyzed by considering specific values for the parameters associated with the step size. Several studies, including \cite{li2021second,ghadimi2013stochastic,wang2021convergence}, have explored and derived convergence rates for the SGD algorithm by imposing constraints or assumptions on the initial step size or the parameters related to the step size. These conditions play a crucial role in ensuring the convergence properties of the algorithm. Table \ref{complexity} presents the convergence rates of various decay step sizes. Specifically, the second column of the table indicates the conditions that need to be satisfied by the initial step size $\eta_0$ or the parameters associated with the step size in order to achieve the desired convergence properties. As presented in Table \ref{complexity}, 
to achieve a convergence rate of order $O(\frac{1}{\sqrt{T}})$, Li et al. \cite{li2021second} made the assumption $\eta_0=\frac{1}{Lc}$ and $c\propto O({\sqrt{T}})$. This assumption implies that the value of the initial step size, denoted as $\eta_0$, will be very small. It is evident that for large values of $T$, the initial step size will be extremely tiny, potentially impacting the algorithm's efficiency.}

{
\begin{table*}[t] 
\centering
\begin{tabular}{ |p{4cm}|p{3.5cm}|p{3cm}|c| }
 \hline
 Step sizes & Conditions  & Convergence rate&Reference\\
  \hline
$\eta_t=\frac{\eta_0}{2}(1+\cos(\frac{\pi t}{T}))$ &   $\eta_0=\frac{1}{Lc},~c=O({\sqrt{T}})$ & $O(\frac{1}{\sqrt{T}})$ &\cite{li2021second}\\
 $\eta_t=\eta_0(\frac{\beta}{T})^{\frac{t}{T}}$& $\eta_0=\frac{1}{Lc},~c=O({\sqrt{T}})$& $O(\frac{1}{\sqrt{T}})$&\cite{li2021second}\\
  $\eta_t=\eta_0(\frac{\beta}{T})^{\frac{t}{T}}$ & $\beta=O(\sqrt{T})$ & $O(\frac{\ln T}{\sqrt{T}})$&\cite{wang2021convergence}\\
$\eta_t$=Constant&  $\eta_0=\frac{1}{\sqrt{T}}$  & $O(\frac{1}{\sqrt{T}})$&\cite{ghadimi2013stochastic}\\
 New proposed method & $\eta_0=\frac{1}{Lc},~c=O(\frac{\sqrt{T}}{\textcolor{red}{\ln T}})$ & $O(\frac{1}{\sqrt{T}})$&New \\
 \hline
\end{tabular}
    \caption{Convergence rates of SGD based on the specific values for the initial step size or the parameters associated with the step size. }
    \label{complexity}
\end{table*}}
\subsection{Contribution}
{
Motivated by the aforementioned studies, our work introduces a novel logarithmic step size for stochastic gradient descent (SGD) with the warm restarts technique. The main contributions of this paper can be summarized as follows:}
\begin{itemize}
\item {The new proposed step size offers a significant advantage over the cosine step size \cite{li2021second} in terms of its probability distribution, denoted as $\frac{\eta_t}{\sum_{t=1}^T\eta_t}$ in Theorem \ref{comp}. This distribution plays a crucial role in determining the likelihood of selecting a specific output during the iterations. Fig. \ref{fig:step_size2} (Left) illustrates that the cosine step size assigns a higher probability of selection to the initial iterations but substantially reduces the probability for the final iterations, approaching zero. In contrast, the new step size proves to be more effective for the final iterations, as they enjoy a higher probability of selection compared to the cosine step size. Consequently, the new step size method outperforms the cosine step size method when it comes to the final iterations, benefiting from their increased likelihood of being chosen as the selected solution.
\item For the new step size, we establish the convergence results of the SGD algorithm. By considering that $c\propto O(\frac{\sqrt{T}}{\ln T})$, which leads to the initial value of the step size is greater than the initial value of the step length mentioned in \cite{li2021second}, we demonstrate a convergence rate of $O\left(\frac{1}{\sqrt{T}}\right)$ for a smooth non-convex function. This convergence rate meets the best-known rate of convergence for smooth non-convex functions.}
\item We conduct a comparative experimental analysis of the new logarithmic step size to nine popular step sizes, including Armijo line search \cite{vaswani2019painless}, cosine step size \cite{li2021second}, Adam, $\frac{1}{t}$, $\frac{1}{\sqrt{t}}$, constant step size, reduceLROnPlateau, stagewise - 1 milestone, stagewise - 2 milestones. We compare the performance of our proposed approach with that of the state-of-the-art methods on the FashionMinst, CIFAR10, and CIFAR100 datasets. 
We demonstrate the effectiveness of the newly proposed step size in enhancing the accuracy and training loss of the SGD on popular datasets such as FashionMNIST, CIFAR10, and CIFAR100 (e.g. see Fig. \ref{fig:step_size2} (Right)). Notably, we observe a significant accuracy improvement of $0.9\%$ specifically for the CIFAR100 dataset when we use a convolutional neural network model.
\end{itemize}

The remaining sections of this paper are structured as follows: In Section 2, we introduce the new logarithmic step size and discuss its properties. Section 3 is dedicated to the analysis of convergence rates for the proposed step size on smooth non-convex functions. We establish that the SGD method achieves an impressive convergence rate of $O(\frac{1}{\sqrt{T}})$. In Section 4, we present and discuss the numerical results obtained using the new decay step size. Finally, in Section 5, we summarize our findings and draw conclusions based on our study.
\begin{figure*}[t]
  \centering  
  \includegraphics[width=0.29\textwidth]{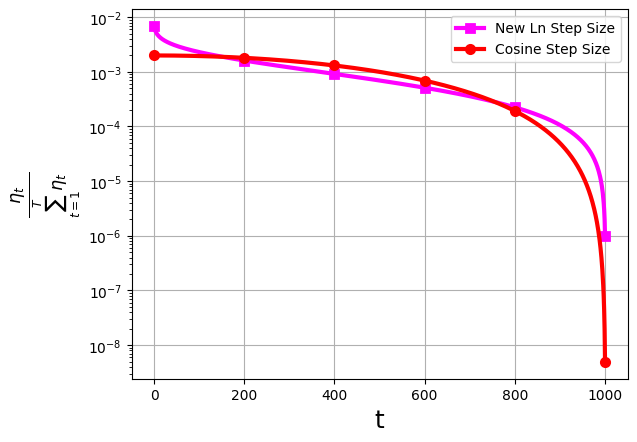}
  \includegraphics[width=0.3\textwidth]{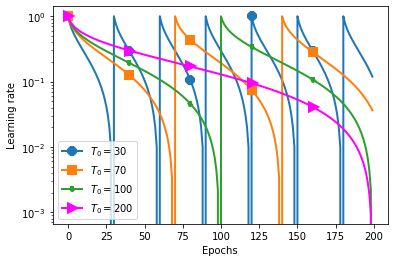}
  \includegraphics[width=0.3\textwidth]{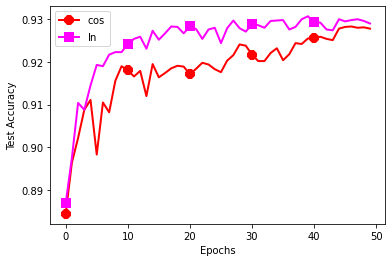}
  \caption{ Left:  Representation of the value $\frac{\eta_t}{\sum_{t=1}^T\eta_t}$ for both the new step size and cosine  step size in \cite{li2021second}. Middle: Warm restarts simulated every $T = 30$ (blue line), $T =70$ (orange line), $T=100$ (green line), and $T=200$ (magenta line) epochs with $\eta_0=1$. Right: The comparison of the warm restart  Algorithm \ref{alg1} with the new proposed step size and cosine step size on the FashionMnist dataset. }
  \label{fig:step_size2}
\end{figure*}

{
In this paper, we use the following notational conventions:
\begin{itemize}
    \item The Euclidean norm of a vector is denoted by $\|.\|$.
\item The nonnegative  {orthant} and positive orthant of $\Bbb{R}^n$ are denoted by $\Bbb{R}+^n$ and $\Bbb{R}{++}^n$, respectively.
\item We use the notation $f(t)=O(g(t))$ to indicate that there exists a positive constant $\omega$ such that $f(t)\leq \omega g(t)$ for all $t\in \Bbb{R}_{++}$.
\end{itemize}}
\section{Problem Set-Up}
We consider the following optimization  problem:  \begin{equation}\label{cost-function}
\min_{x\in \Bbb{R}^d}f(x)=\min_{x\in \Bbb{R}^d} \frac{1}{n}\sum_{i=1}^{n}f_i(x),
\end{equation}
where $f_i :~\Bbb{R}^d\rightarrow \Bbb{R}$ is the loss function for the $i$-th training sample over the variable $x\in\Bbb{R}^d$ and $n$ denotes the number of samples. This minimization problem is central in machine learning. {
Several iterative approaches for solving equation (1) are known \cite{Nocedal}, and SGD is particularly popular when the dimensionality, $n$, is extremely large \cite{robbins51,nemirovski09}}. SGD uses a random training sample $i_k \in \{1,2,...,n\}$ to update $x$ using the rule:
\begin{equation}\label{new_point}
    x_{k+1}=x_k-\eta_k\nabla f_{i_k}(x_k),
\end{equation}
in which $\eta_k$ is the step size used in iteration $k$ and $\nabla f_{i_k}(x)$  is the (average) gradient of the loss function(s) \cite{vaswani2019painless}.
\\
Our analysis and results in this paper are based on the following assumptions.
\begin{itemize}
    \item {\bf{Assumption 1.}} The function $f$ is differentiable and L-smooth, that is, for all $x, y\in \Bbb{R}^d$ there exists
a constant $L > 0$ such that:
\begin{eqnarray*}
   & \|\nabla f(x)-\nabla f(y)\|\leq L\|x-y\|,&
\\
        &f(y)\leq f(x)+\langle \nabla f(x),y-x \rangle+\frac{L}{2}\|y-x\|^2.&
    \end{eqnarray*}
 \item {\bf{Assumption 2.}} {For any random $t\in [1,T]$,} we have
\begin{equation}
    \Bbb{E}_t\left[\|g_t-\nabla f(x_t)\|^2\right]\leq \sigma^2.
\end{equation}
 \item {\bf{Assumption 3.}} The objective function $f$ is bounded below on $\Bbb{R}^d$.
\end{itemize}
\subsection{The New Step Size}
It has been observed that larger values of the step size provide the model with sufficient energy to escape from critical points in the initial iterations, while smaller step sizes guide the model towards well-behaved local minima in the final iterations {\cite{mishra2019polynomial}}. {However, when the step length of a function takes on large values over multiple iterations and rapidly tends to zero towards the end, it adversely affects the behavior of the distribution function denoted as $\frac{\eta_t}{\sum_{t=1}^T\eta_t}$. Consequently, based on the conditions described in Theorem \ref{comp}, the probability of selecting points from the final iterations decreases significantly compared to the initial iterations.
\\
To address these issues, we propose an appropriate step size that its values gradually decrease during the iterations. In this context, we introduce a logarithmic step size that exhibits slower convergence to zero compared to many other step sizes, yet converges faster than the cosine step size.} 
The new step size is defined as follows:
\begin{definition}
For a constant $\eta_0$, the new logarithmic step size for SGD is given by:
\begin{equation}\label{new_step_size}
    \eta_t=\eta_0\left(1-\frac{\ln(t)}{\ln T}\right),\quad~ t \in [1,T],
\end{equation}
\end{definition}

where $T$ represents the number of iterations in each inner loop of the algorithm and $t$ represents the number of epochs performed since the last restart. To use the warm restarts policy, we follow the model proposed in \cite{smith2017cyclical}, which states that the training process is divided into $l$ cycles (i.e., the outer loop of Algorithm \ref{alg1}), and the algorithm runs each cycle in $T$ epochs. Each cycle begins with the largest value of step size $\eta_0$. During the running of a cycle, the value of the learning rate decreases by using (\ref{new_step_size}). In addition, when $t=T$ the step size $\eta$ will be output, i.e., $\eta_T=0$. The behaviours of the new proposed step size for different values of $T$ are shown in Fig. \ref{fig:step_size2} (Middle). It is clear that for small values of $T$, that is, $T=\{30,70\}$, the algorithm performs more inner loop, and the values of step size quickly go to zeros, but when $T=200$, the algorithm performs one inner loop and the step size goes to zeros very slowly.
\\
The lemma that follows gives us some properties of the function $\ln(x)$. We will utilize them frequently in the rest of the paper.
\begin{lemma}\label{ln_pro}
    For the function $\ln (x)$, we have:
    \begin{itemize}
    \item $\ln xy=\ln x+\ln y$
    \item $\ln\frac{x}{y}=\ln x-\ln y$.
   \item $\frac{x-1}{x}\leq \ln (x)$, for all $x\geq 1$.
   \end{itemize} 
\end{lemma}
\section{Algorithm and convergence }
We use the warm restart Algorithm \ref{alg1} with condition { $T_0=T_1=...=T_l$
which means the algorithm runs with the same number of epochs in the inner loop.} Algorithm \ref{alg1} begins with the given initial step size $\eta_0$, the number of inner iterations $T$, the number of outer epochs $l$, and the starting point $x^1_0$. The algorithm consists of two loops, i.e., outer and inner loops. In each inner loop, the SGD with the logarithmic step size is performed.
\\
\begin{algorithm*}[t]
\caption{SGD with Warm Restarts Logarithmic Step Size}
\label{alg1}
 {Input:}~Initial Step size $\eta_0$, initial point $x^1_0$, the number of outer and inner iterations, i.e., $l$ and $T$.\\
 { Set:} $T_0=T$\\
\For {$i=0,...,l$}{
 \For {$t=1,...,T$}{
 Set $\eta_t = \eta_0 \left(1 - \frac{\ln t}{\ln T}\right)$ and update $x_i^{t+1}$ using (\ref{new_point}).
\\
 End}
 $x^1_{i+1}=x_i^{T}$\\
  End}
$x^*=x^T_l$
\end{algorithm*}
\subsection{Convergence results for smoothness function}
We present the convergence bound for SGD with logarithmic step sizes when the objective function is a smooth and non-convex function. The lemma that follows provides a lower bound for {$\sum_{t=1}^T\eta_t$.}
\begin{lemma}\label{low_bound}
For the new proposed step size given by (\ref{new_step_size}), we have:
\begin{equation}
\sum_{t=1}^{T}\eta_{t}\geq \eta_0\frac{T+1}{2\ln T}.\nonumber 
\end{equation}
\end{lemma}
The next lemma gives an upper bound for $\sum_{t=1}^T\eta_t^2$.
\begin{lemma}\label{uper_bound}
For the step size given by (\ref{new_step_size}), we have:
\begin{equation}
   {\sum_{t=1}^T\eta_t^2\leq \eta_0^2\frac{2T}{\ln^2 T}.}
\end{equation}
\end{lemma}
\begin{lemma}\label{uper_grad}
 Under Assumptions 1 and 2, and $\eta_t \leq \frac{1}{cL}
$. Then, Algorithm 1 with the new proposed step size guarantees:
\begin{equation}
    \frac{\eta_t}{2}\Bbb{E}[\|\nabla f(x_t)\|^2]\leq \Bbb{E}[f(x_t)]-\Bbb{E}[f(x_{t+1})]+\frac{L\eta^2_t\sigma^2}{2}\nonumber
\end{equation}
\end{lemma}
The next theorem gives an upper bound for \\$  \Bbb{E}\|\nabla f(\bar{x}_T)\|^2$.  
\begin{theorem}\label{comp}
 Under Assumptions 1 and 2, and $c>1$. SGD with the new proposed step sizes with $\eta_0=\frac{1}{cL}$ guarantees:
\begin{eqnarray}
  \Bbb{E}\|\nabla f(\bar{x}_T)\|^2\nonumber\leq \frac{4cL\ln T}{(T+1)}\left(f(x_1)-f^*\right)+\frac{4\sigma^2T}{Lc(T+1)\ln T}\nonumber
\end{eqnarray}
where $\bar{x}_T$ is a random iterate drawn from the sequence $\{x_t\}_{t=1}^T$ with probability $\Bbb{P}[\bar{x}_T=x_t]=\frac{\eta_t}{\sum_{t=1}^T\eta_t}$.
\end{theorem}
\begin{proof}
Using the definition of $\bar{x}_T$ and Lemma \ref{uper_grad}, we have:
\begin{eqnarray}
  \Bbb{E}[\|\nabla f(\bar{x}_T)\|^2]=\frac{\eta_t\Bbb{E}[\|\nabla f(x_t)\|^2]}{\sum_{t=1}^T\eta_t}
     &\leq&\frac{2\sum_{t=1}^T\left[\Bbb{E}[f(x_t)]-\Bbb{E}[f(x_{t+1})]\right]}{\sum_{t=1}^T\eta_t}
     +\frac{L\sigma^2\sum_{t=1}^T\eta_t^2}{\sum_{t=1}^T\eta_t}\nonumber\\
      &\leq&\frac{2\left(f(x_1)-f^*\right)}{\sum_{t=1}^T\eta_t}+\frac{L\sigma^2\sum_{t=1}^T\eta_t^2}{\sum_{t=1}^T\eta_t}
      \leq\frac{4\ln T\left(f(x_1)-f^*\right)}{\eta_0(T+1)}+\frac{4L\sigma^2\eta_0T}{(T+1)\ln T}\nonumber\\
          &=&\frac{4cL\ln T}{(T+1)}\left(f(x_1)-f^*\right)+\frac{4\sigma^2T}{Lc(T+1)\ln T}.\nonumber
\end{eqnarray}
where the third inequality is obtained by using Lemmas \ref{low_bound} and \ref{uper_bound}. 
\end{proof}
{
As mentioned earlier, selecting a smaller value for the parameter $c$ compared to the cosine function used in \cite{li2021second} enables us to achieve the optimal convergence rate for SGD based on the new logarithmic step size. Specifically, when considering $c \propto O\left(\frac{\sqrt{T}}{\ln T}\right)$, we have the following result:}
\begin{corollary} Based on the Theorem \ref{comp}, if $\sigma\neq 0$, setting $c\propto \frac{\sqrt{T}}{\ln T} $, we have:
    \begin{eqnarray}
       \Bbb{E}\|\nabla f(\bar{x}_T)\|^2\leq \frac{4L}{\sqrt{T}}\left(f(x_1)-f^*\right)+\frac{4\sigma^2}{L\sqrt{T}} \nonumber
    \end{eqnarray}
\end{corollary}

{
Now, leveraging the results obtained from Theorem \ref{comp}, we can compute the convergence rate for the warm restart SGD algorithm.}
\begin{corollary}\label{warm} {(SGD with warm restarts): Under Assumptions A1, A2, and A3, for a given value of $T$, where $\sum_{i=1}^{l}T_i=lT$ and $\eta_0=\frac{1}{cL}$, Algorithm \ref{alg1} ensures the following convergence guarantees:}
    \begin{equation}
    \Bbb{E}\|\nabla f(\bar{x}_T)\|^2\leq \frac{4lcL\ln T}{T}\left(f(\tilde{x}_{1})-f^*\right)+\frac{4\sigma^2l}{Lc\ln T},\nonumber
\end{equation} 
where $f(\tilde{x}_1)=\max_i\{f(x_{1_i})\}$ for $i=1,2,...,l$.
\end{corollary}
\section{Numerical Results}
\label{sec:numerical-results}
\subsection{Experiments on MNIST,  CIFAR10 and CIFAR100}
In this section, we evaluate our proposed algorithm's performance on image classification tasks. We compare the performance of the proposed approach with that of the state-of-the-art methods on the FashionMinst, CIFAR10, and CIFAR100 datasets {\cite{krizhevsky2009learning}}. FashionMinst is a dataset consisting of a training set of 60000 examples of grayscale images and a test set of 10000 examples. Each image is $28*28$ pixels. We use the convolutional neural network (CNN) model for the classification task on this dataset. This model has two convolutional layers with kernel sizes of $5\times 5$ and padding of $2$, two max-pooling layers with kernel sizes of $2\times 2$, and two fully connected layers with $1024$ hidden nodes. The activation function of the hidden nodes is the rectified linear unit (ReLU). To prevent overfitting, we use a dropout with a probability of $0.5$ in the hidden layer of the deep model. The number of input neurons is $746$, and the number of output neurons is $10$. To compare the performance of various algorithms, we use the cross-entropy function as a loss function and an accuracy metric. 

The CIFAR10 dataset consists of 60,000 of $32\times 32$ color images of 10 classes, each with 6000 images. There are 50,000 training images and 10,000 test images. The batch size is $128$. It means that each epoch of the training includes $390$ iterations. The deep learning architecture used for evaluating the performance of the algorithms on this dataset is the $20$-layer residual neural network (ResNet) \cite{he2016deep}. This model uses cross-entropy as the loss function.  

The CIFAR100 dataset is just like the CIFAR10, except it has 100 classes containing 600 natural images each. There are 500 training images and 100 testing images per class. There are $50,000$ images available for training and $10,000$ images available for testing. Randomly cropped and flipped images are used for training. {We have conducted two series of experiments on this dataset based on two deep learning models. The first one is the same as that used on FashionMinst which is described above. The reason for using this model may be the limitation in the hardware available to run the experiments. The second one is a DenseNet-BC model with $100$ layers and a growth rate of $12$ \cite{huang2017densely}. } 

{
To optimize the hyperparameter $\eta_0$ of the new proposed step decay on the CIFAR100 dataset, we employed a two-stage grid search on the validation set. Initially, we explored a rough grid and picked the hyperparameters that produced the best validation outcomes. Subsequently, we conducted a more refined search centered around the most effective hyperparameters identified in the first stage and chose the optimal set of hyperparameters as the final selection. For the starting step size $\eta_0$, we used a coarse search grid of $\{0.00001, 0.0001, 0.001, \\ 0.01, 0.1, 0.2,0.3,0.4,0.5,0.6,0.7, 0.8, 0.9,1\}$, and a fine grid of $\{0.41, 0.42, 0.43, 0.44, 0.45, 0.46, 0.47, \\ 0.48, 0.49, 0.51, 0.52, 0.53, 0.54, 0.55, 0.56, 0.57, \\ 0.58, 0.59\}$. The best value of $\eta_0$ was obtained as $0.53$.
}
{
We compared the performance of our proposed method with state-of-the-art methods whose hyperparameter values were fine-tuned in a previous study \cite{li2021second}. We adopted the same hyperparameter values in our own numerical studies. 
The values of the hyperparameters are given in Table \ref{Table_Hyperparam}. In the experiments on the CIFAR100 dataset using Convolutional Neural Network (CNN), the parameter $\eta_0$ is set $0.09$, which is similar to the cosine step decay method. The momentum parameter is consistently set to $0.9$ across all methods and datasets mentioned. 
Additionally, weight-decay is set to $0.0001$ for FashionMnist and CIFAR10, and $0.0005$ for CIFAR100 in all mentioned methods. Furthermore, a batch size of 128 is used in all of the experiments mentioned. The parameter $T_0$ is the ratio of training samples to the batch size.}
The experiments are repeated with random seeds for $5$ times to eliminate the influence of stochasticity. 

\begin{table*}[t] 
\centering
\begin{tabular}{ |p{4cm}|p{1cm}|p{1cm}|p{0.4cm}||p{1 cm}|p{1cm}|p{0.4cm}||p{1 cm}|p{1cm}|p{0.4cm}|}
 \hline
 & \multicolumn{3}{|c||}{FashionMNIST}& \multicolumn{3}{|c||}{CIFAR10}& \multicolumn{3}{|c|}{CIFAR100}\\
 \hline
 Methods & $\eta_0$ & $\alpha$  & $c$ & $\eta_0$  & $\alpha$  & $c$ & $\eta_0$ & $\alpha$  & $c$\\
  \hline
Constant Step Size &   $0.007$  & - & - &   $0.07$ & - & - & $0.07$ & - & -\\
 $O(\frac{1}{t})$ Step Size & $0.05$ & $0.00038$ & -& $0.1$ & $0.00023$ & -& $0.8$ & $0.004$ & -\\
 $O(\frac{1}{\sqrt{t}})$ Step Size & $0.05$ & $0.00653$ & -& $0.2$  & $0.07907$ & - & $0.1$  & $0.015$ & - \\
 Adam &   $0.0009$  & - & -& $0.0009$  & - & - & $0.0009$  & - & - \\
 SGD+Armijo & $0.5$  & - & $0.1$ & $2.5$  & - & $0.1$ & $5$  & - & $0.5$\\
 ReduceLROnPlateau & $0.04$  & $0.5$ & - & $0.07$   & $0.1$ & - & $0.1$   & $0.5$ & - \\
 Stagewise - 1 Milestone & $0.04$   & $0.1$ & - & $0.1$ & $0.1$ & -& $0.07$ & $0.1$ & -\\
 Stagewise - 2 Milestones & $0.04$ & $0.1$ & -& $0.2$  & $0.1$ & -& $0.07$  & $0.1$ & -\\
 Cosine step size & $0.05$ & - & -& $0.25$  & - & -& $0.09$  & - & -\\
 Ln step size & $0.05$ & - & -& $0.25$ & $-$ & -& $0.53$ & $-$ & -\\
 \hline
\end{tabular}
    \caption{The hyperparameter values for the methods employed on the FashionMnist, CIFAR10, and CIFAR100 datasets.}
    \label{Table_Hyperparam}
\end{table*}
\subsection{Methods }
We consider SGD with the following step sizes:
\begin{itemize}
    \item $\eta_t=constant$
    \item $\eta_t=\frac{\eta_0}{1+\alpha\sqrt{t}}$
     \item $\eta_t=\frac{\eta_0}{1+\alpha t}$
     \item $\eta_t=\frac{\eta_0}{2}\left(1+\cos\frac{t\pi}{T}\right)$
     \item $\eta_t=\eta_0\left(1-\frac{\ln t}{\ln T}\right)$
\end{itemize}
which are respectively named with SGD constant step size, $O(\frac{1}{\sqrt{t}})$ step Size, $O(\frac{1}{t})$ step Size, cosine, and the new logarithmic step size. Parameter $t$ represents the iteration number of the inner loop and each outer iteration consists of iterations for training on mini-batches. 
Moreover,  we compare the results  of  the new proposed logarithmic step size with  Adam \cite{kingma2014adam},
SGD+Armijo \cite{vaswani2019painless}, PyTorch’s ReduceLROnPlateau scheduler5 (abbreviated as ReduceLROnPlateau), and stagewise step size. {Note that, similar to \cite{li2021second}, the term "stagewise" refers to the Stagewise - 1 Milestone and Stagewise - 2 Milestone methods.} (As a side note, since we use Nesterov momentum in all SGD variants, the stagewise step decay basically covers the performance of multistage accelerated algorithms (e.g., \cite{aybat2019universally}). In all experiments, we follow the setting proposed in \cite{li2021second}.

\begin{figure*}
  \centering
  \includegraphics[width=1\textwidth]{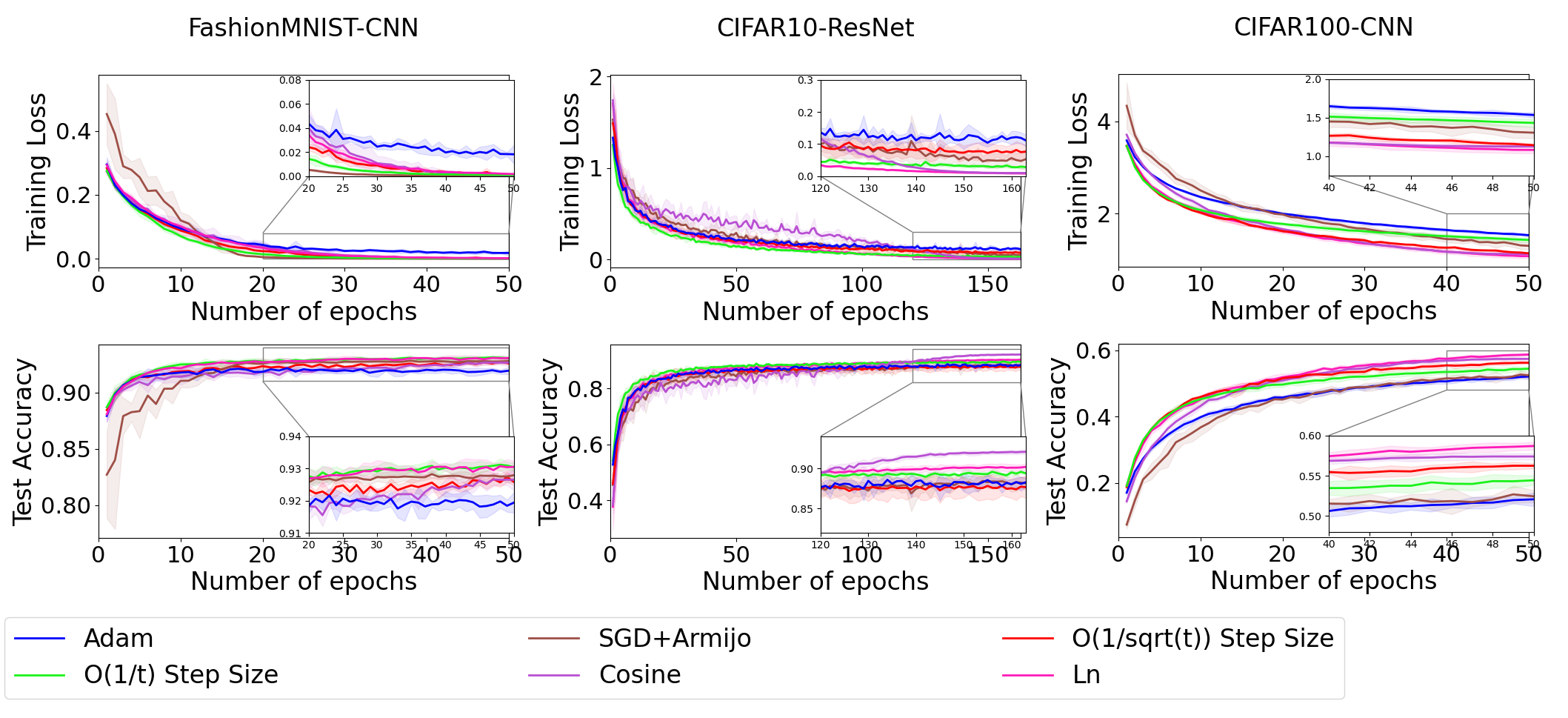}
  \caption{Comparison of new proposed step size and five other step sizes on FashionMinst, CIFAR10, and CIFAR100 datasets. }
  \label{fig:CIFAR101}
\end{figure*}


\subsection{Results and Discussion}
We compare the methods in two groups as illustrated in Figures \ref{fig:CIFAR101} and \ref{fig:CIFAR102} for more clarity of figures. In Figure \ref{fig:CIFAR101}, the new proposed step size achieves the training loss close to zero as the best-known method such as SGD+Armijo after epoch number 40 in the FashionMnist dataset while its test accuracy outperforms all methods. In the CIFAR10 dataset, the new logarithmic step size is as good as the best previously studied method (i.e. $O(\frac{1}{t})$ step size). Figure \ref{CIFR100} demonstrates that the new proposed step size has the best results in terms of test accuracy for the FashionMnist dataset. In the CIFAR10 dataset, the new proposed approach outperforms all the other methods both based on the training loss and the test accuracy. On the other hand, the new logarithmic step size achieves the best training loss, but the stagewise step decay step size works well in terms of accuracy than the other step size for the CIFAR10 dataset as illustrated in Figure \ref{fig:CIFAR102}. 

{
As previously mentioned, we assessed the performance of the new proposed step size on the CIFAR100 dataset using two deep models for image classification. The results, depicted in Figures \ref{fig:CIFAR101} and \ref{fig:CIFAR102}, demonstrate that the new proposed step size outperforms all other methods in both groups when the deep model is a convolutional neural network (CNN). Overall, the use of the DenseNet-BC model results in better performance for all step decays, as shown in Figures \ref{fig:CIFAR10021} and \ref{fig:CIFAR10022}. Notably, the new proposed step size performs better than the state-of-the-art method (i.e. cosine step decay).}

Tables \ref{fashion}--\ref{CIFR100} demonstrate the average of the final training loss and test accuracy obtained by 5 runs starting from different random seeds on FashionMNIST, CIFAR10, and CIFAR100 datasets, respectively.
\\
From Tables \ref{fashion}--\ref{CIFR100}, we can conclude that:
\begin{itemize}
    \item From Table \ref{fashion}, the SGD+Armijo step size obtained the best training loss, however, the proposed step size improved the test accuracy of the SGD by $0.03\%$.
    \item Table \ref{CIFR10} shows that the new proposed step size has the best training loss and the cosine step size has the best test accuracy.
    \item From Table \ref{CIFR100} we can conclude that the new proposed step size can improve the test accuracy of the SGD  by $0.9\%$. 
    \item {Based on the information presented in Table \ref{CIFR100-dens}, it can be inferred that the new proposed step size leads to a $0.2 \%$ improvement in training loss and a $0.11\%$ improvement in test accuracy compared to the cosine step decay method.}
\end{itemize}

\begin{figure*}[t]
  \centering
    \includegraphics[width=1\textwidth]{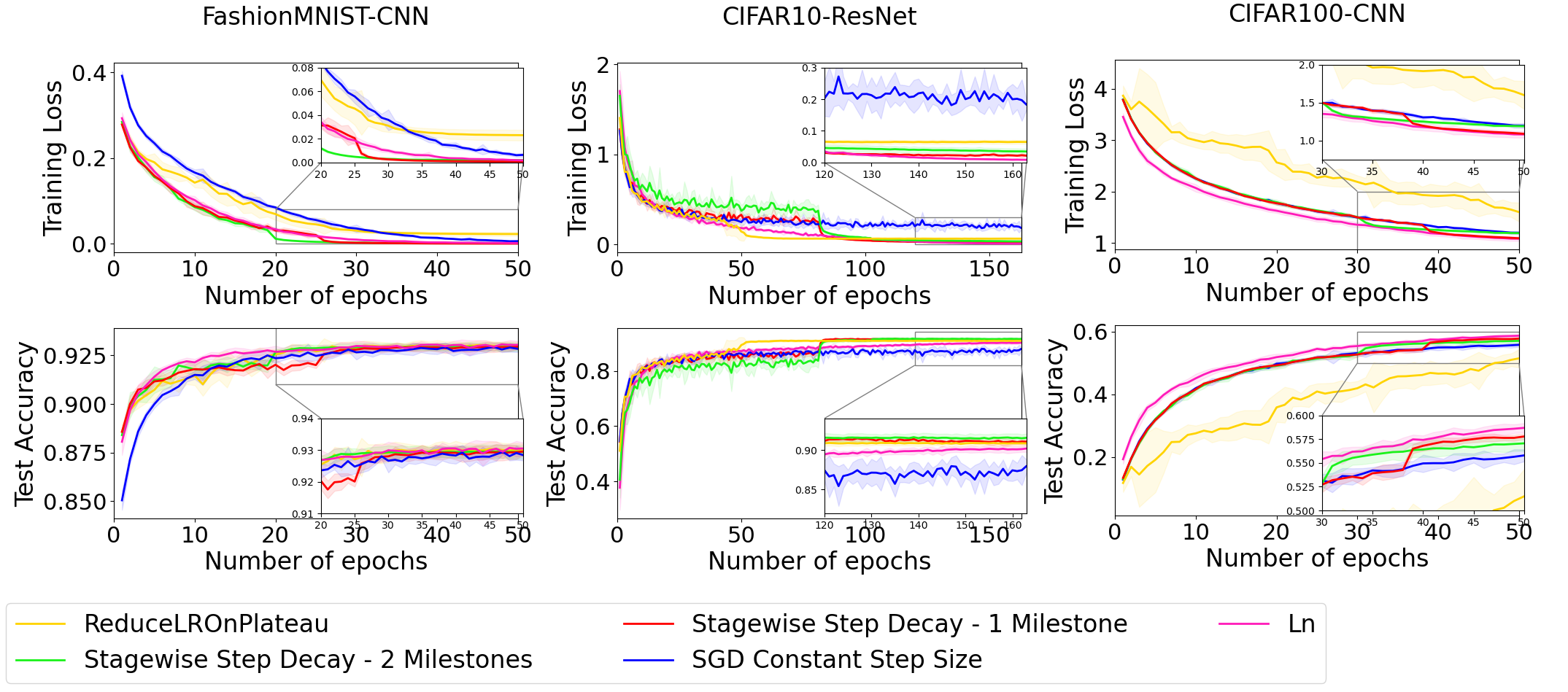}
  \caption{Comparison of new proposed step size and four other step sizes on FashionMinst, CIFAR10, and CIFAR100 datasets.}
  \label{fig:CIFAR102}
\end{figure*}
\begin{figure*}[t]
  \centering
    \includegraphics[width=0.8\textwidth]{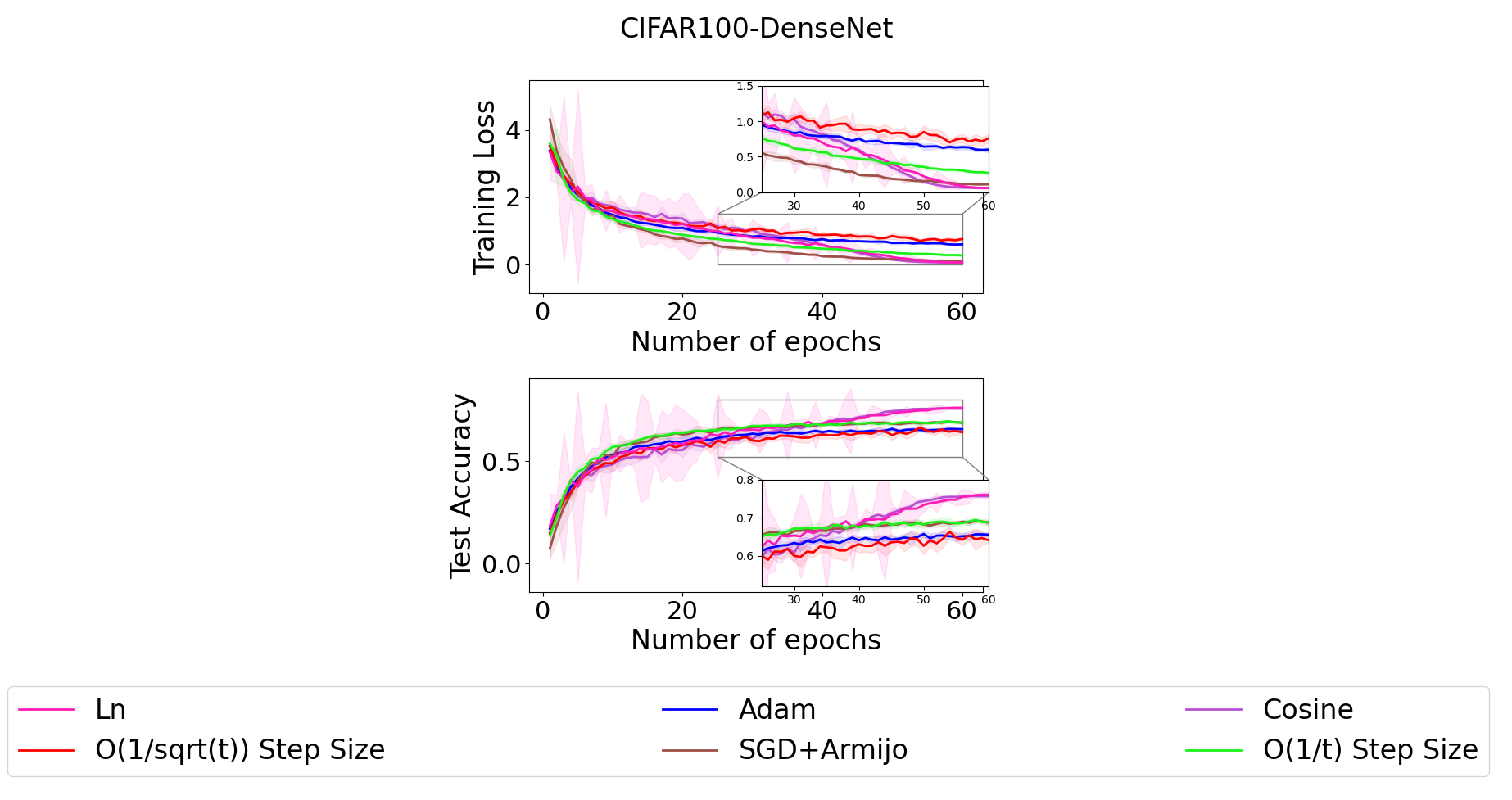}
  \caption{Comparison of new proposed step size and five other step sizes on CIFAR100 dataset using the DenseNet-BC model.}
  \label{fig:CIFAR10021}
\end{figure*}
\begin{figure*}[t]
  \centering
    \includegraphics[width=0.8\textwidth]{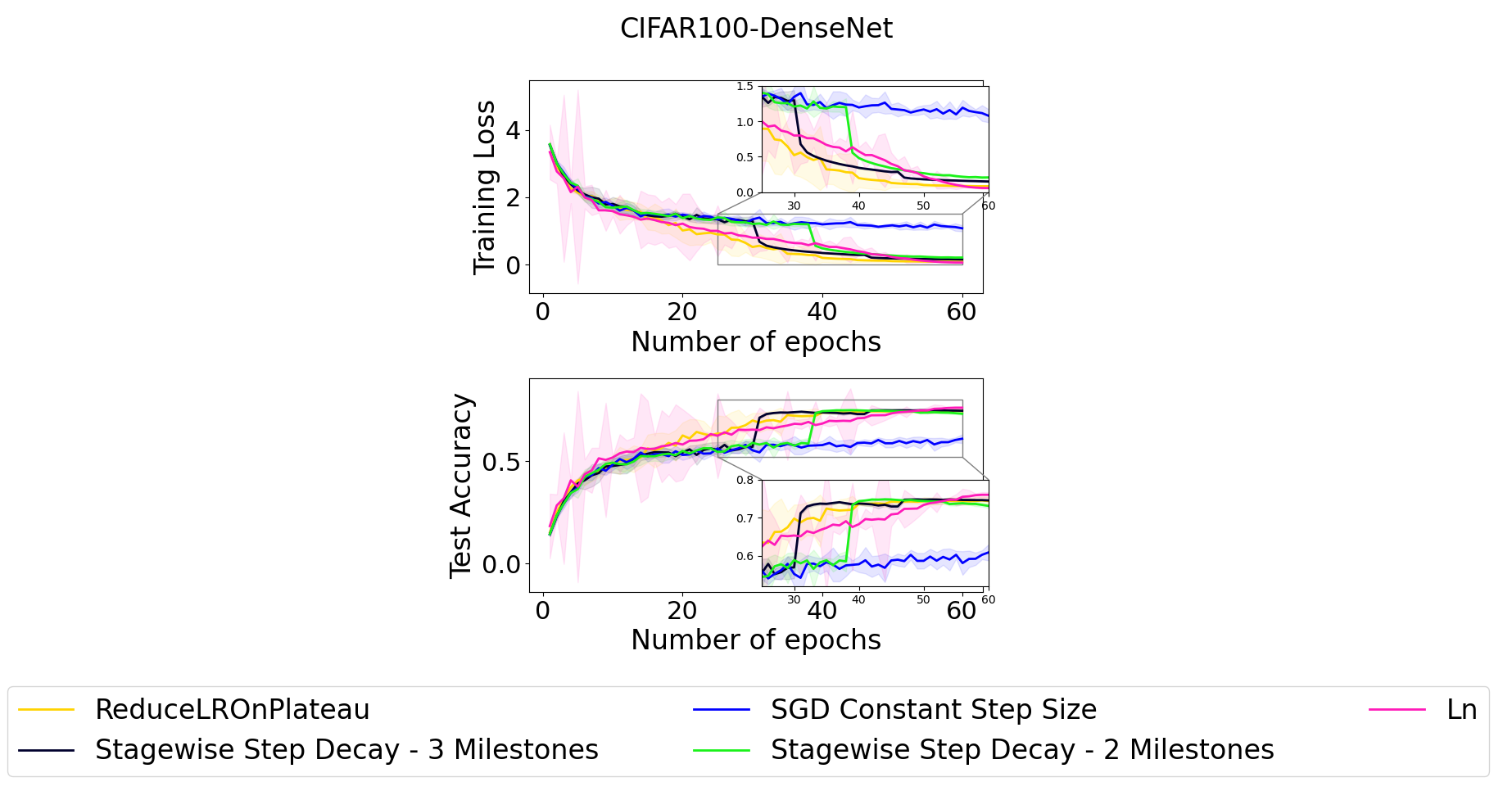}
  \caption{Comparison of new proposed step size and four other step sizes on CIFAR100 dataset using the DenseNet-BC model.}
  \label{fig:CIFAR10022}
\end{figure*}

\begin{table*}[t] 
\centering
\begin{tabular}{ |p{4cm}|p{3.5cm}|p{3cm}| }
 \hline
 Methods & Training loss   & Test accuracy \\
  \hline
Constant Step Size &   $0.0007\pm 0.0003$ & $0.9299\pm 0.0016$
 \\
 $O(\frac{1}{t})$ Step Size & $0.0006 \pm 0.0000$ & $0.9311 \pm 0.0012 $\\
 $O(\frac{1}{\sqrt{t}})$ Step Size & $0.0011\pm 0.0003$ & $0.9275
 \pm 0.0007$\\
 Adam &   $0.0131 \pm 0.0017$  & $0.9166 \pm 0.0019 $\\
 SGD+Armijo & \bf{6.73E-05 $\pm$ 0.0000} & $0.9277 \pm 0.0012$ \\
 ReduceLROnPlateau & $0.0229 \pm 0.0032$  & $0.9299\pm 0.0014
$\\
 Stagewise - 1 Milestone & $0.0004 \pm 0.0000$  & $0.9303 \pm 0.0007$\\
 Stagewise - 2 Milestones & $0.0015 \pm 0.0000$  & $0.9298\pm 0.0014$  \\
 Cosine step size& $0.0004 \pm 1.1E-05$ & $0.9284 \pm 0.0005$  \\
 Ln step size & $0.0007\pm 0.0000$ & \bf{0.9314 $\pm$ 0.0018} \\
 \hline
\end{tabular}
    \caption{Average final training loss and test accuracy on the FashionMNIST dataset. 
 The $95\%$ confidence intervals of the mean loss and accuracy value over 5 runs starting from different random seeds }
    \label{fashion}
\end{table*}

\begin{table*}[t] 
\centering
\begin{tabular}{ |p{4cm}|p{3.5cm}|p{3cm}| }
 \hline
 Methods & Training loss   & Test accuracy \\
  \hline
 Constant Step Size &   $0.1870 \pm 0.0257$ & $0.8776 \pm 0.0060$ \\
 $O(\frac{1}{t})$ Step Size & $0.0261 \pm 0.0021$ & $0.8946
 \pm 0.0020$\\
 $O(\frac{1}{\sqrt{t}})$ Step Size & $0.2634 \pm 0.3493$ & $0.8747 \pm 0.0077$\\
 Adam &   $0.1140 \pm 0.0060$  & $0.8839\pm 0.0031$\\
 SGD+Armijo & $0.0463 \pm 0.0169$ & $0.8834 \pm 0.0059$ \\
 ReduceLROnPlateau & $0.0646 \pm 0.0041$  & $0.9081\pm 0.0019$\\
 Stagewise - 1 Milestone & $0.0219 \pm 0.0012$  & $0.9111 \pm 0.0034$  \\
 Stagewise - 2 Milestones & $0.0344\pm 0.0062$  & $0.9151\pm 0.0039$  \\
 Cosine step size & $0.0102\pm 0.0007$ & \bf{0.9201 $\pm $ 0.0017} \\
 Ln step size  &  \bf{0.0089 $\pm$ 0.0011} & $0.9012 \pm 0.0027$  \\
 \hline
\end{tabular}
    \caption{Average final training loss and test accuracy obtained by 5 runs starting from different random seeds and $\pm$ is followed
by an estimated margin of error under $95\%$ confidence on the CIFR10 dataset.}
    \label{CIFR10}
\end{table*}

\begin{table*}[t] 
\centering
\begin{tabular}{ |p{4cm}|p{3.5cm}|p{3cm}| }
 \hline
 Methods & Training loss   & Test accuracy \\
  \hline
 Constant Step Size &   $1.19599 \pm 0.0183$ & $0.5579\pm 0.0039$ \\
 $O(\frac{1}{t})$ Step Size & $1.43381 \pm 0305$ & $0.54436 \pm 0.0038$\\
 $O(\frac{1}{\sqrt{t}})$ Step Size & $1.14579 \pm 0.0125$ & $0.5624 \pm 0.0020$\\
 Adam &   $1.53433 \pm 0.0180$  & $0.5206 \pm 0.0053$\\
 SGD+Armijo & $1.3046 \pm 0.0476$ & $0.5244 \pm 0.0053$ \\
 ReduceLROnPlateau & $1.6000 \pm 0.055$  & $0.5148 \pm 0.1929$\\
 Stagewise - 1 Milestone & $1.1889 \pm 0.0077$  & $0.5706 \pm 0.0049$   \\
 Stagewise - 2 Milestones & $1.09061 \pm 0.0040$  & $0.5779 \pm 0.0017$  \\
 Cosine step size & $1.12694 \pm 0.0081$ & $0.5739 \pm 0.0025$  \\
 Ln step size  & \bf{1.0805 $\pm$ 0.0327} & \bf{0.5869 $\pm$ 0.0035}  \\
 \hline
\end{tabular}
    \caption{Average final training loss and test accuracy obtained by 5 runs starting from different random seeds and $\pm$ is followed
by an estimated margin of error under $95\%$ confidence on the CIFR100 dataset using the convolutional neural network (CNN) model. }
    \label{CIFR100}
\end{table*}

\begin{table*}[t] 
\centering
\begin{tabular}{ |p{4cm}|p{3.5cm}|p{3cm}| }
 \hline
 Methods & Training loss   & Test accuracy \\
  \hline
 Constant Step Size &   $1.0789\pm 0.06336$ & $0.6089\pm 0.01485$ \\
 $O(\frac{1}{t})$ Step Size & $0.2709 \pm 0.02047$ & $0.6885 \pm 0.0051$\\
 $O(\frac{1}{\sqrt{t}})$ Step Size & $ 0.7584 \pm 0.04068$ & $0.6411 \pm 0.00469$\\
 Adam &   $0.6008 \pm 0.02396$  & $0.6555 \pm 0.00430$\\
 SGD+Armijo & $0.1162 \pm 0.01175$ & $0.6869 \pm 0.00460$ \\
 ReduceLROnPlateau & $0.0863 \pm 0.00896 $  & $0.7440 \pm 0.00695$\\
 Stagewise - 1 Milestone & $0.15300\pm 0.00382$  & $0.7456 \pm 0.00188$   \\
 Stagewise - 2 Milestones & $ 0.21032\pm 0.01067$  & $0.73102
\pm 0.00402$  \\
 Cosine step size & $0.0622 \pm 0.00349$ & $0.7568 \pm 0.00311$  \\
 Ln step size  & \bf{0.06011 $\pm$ 0.0018} & \bf{0.7579 $\pm$ 0.0022}  \\
 \hline
\end{tabular}
    \caption{Average final training loss and test accuracy obtained by 5 runs starting from different random seeds and $\pm$ is followed
by an estimated margin of error under $95\%$ confidence on the CIFR100 dataset using the DenseNet-BC model. }
    \label{CIFR100-dens}
\end{table*}

\section{Conclusion}
We introduced a new logarithmic step size with warm restarts technique for the SGD. We showed that the new step size goes to zero more slowly than most existing step sizes. We showed that the new logarithmic step size
achieves the rate of $O(\frac{1}{\sqrt{T}})$ for the SGD with a smooth non-convex objective function. Finally, to prove the effectiveness of the new proposed step size in practice, we did a wide range of implementations on three famous datasets, i.e. FashionMinst, CIFAR10, and
CIFAR100 datasets, and compared the obtained results with 9 other step sizes. For the CIFAR100 datasets, the proposed step size improved the accuracy of the algorithm by $0.9\%$ when we utilized a convolutional neural network model. For two other datasets, the new proposed step size obtained good results.
\\
{Further investigation into the convergence rate of SGD based on the logarithmic step size is warranted for convex and strongly convex objective functions. Additionally, while this paper focuses on the convergence rate for smooth non-convex functions that do not satisfy the PL condition, it would be interesting to examine the convergence rate for smooth non-convex functions that satisfy the PL condition.}

\section*{Acknowledgments}
The authors would like to thank the editors and anonymous reviewers for their constructive comments. The research of the first author is partially supported by the University of Kashan under grant number 1143902/2.
\subsection*{Appendix A}
\begin{proof}
The proof of items 1 and 2  can be found in \cite{finney2001thomas}.  
{
To prove item (3), we start by defining the function $h(x)$ and demonstrating its positivity for all $x\geq 1$. Let us define:
\begin{eqnarray}
h(x):=\ln x-\frac{x-1}{x}\nonumber
\end{eqnarray}
Next, we calculate the derivative of $h(x)$:
\begin{equation}
h'(x)=\frac{1}{x}-\frac{1}{x^2}=\frac{x-1}{x^2}
\end{equation}
It is evident that $h'(x)\geq 0$ for all $x\geq 1$, indicating that $h(x)$ is an increasing function for $x\geq 1$. To show that $h(x)$ is positive, we evaluate $h(1)$, which yields $h(1)=0$. Since $h'(x)>0$ and $h(1)=0$, we conclude that $h(x)\geq 0$ for all $x\geq 1$. Consequently, we have $\ln x\geq \frac{x-1}{x}$ for all $x\geq 1$.}

\end{proof}
\begin{proof}{\bf{Proof of Lemma \ref{low_bound}}}
 {  To prove, we use the definition of the new proposed step size and we have:
\begin{eqnarray*}
\sum_{t=1}^{T}\eta_{t}=\eta_0\sum_{t=1}^{T}\left(1-\frac{\ln(t)}{\ln T}\right)=\eta_0\sum_{t=1}^{T}\left(\frac{\ln T-\ln t}{\ln T}\right)
=\frac{\eta_0}{\ln T}\sum_{t=1}^{T}\left(\ln T-\ln t\right)
\end{eqnarray*}
Now, we use the second item of Lemma \ref{ln_pro}, that is, $\ln x-\ln y=\ln\frac{x}{y}$, and we have:
\begin{eqnarray*}
\sum_{t=1}^{T}\eta_{t}&=&\frac{\eta_0}{\ln T}\sum_{t=1}^{T}\ln\frac{T}{t}.
\end{eqnarray*}
Next, we are going to work on $\ln\frac{T}{t}$. In this regard, use the third item of Lemma \ref{ln_pro}, i.e., $\ln x\geq \frac{x-1}{x}$ for all $x\geq 1$.  Since $\frac{T}{t}\geq 1$ therefore, we have  $\ln\frac{T}{t}\geq \frac{\frac{T}{t}-1}{\frac{T}{t}}=\frac{\frac{T-t}{t}}{\frac{T}{t}}=\frac{T-t}{T}$. It implies that:
\begin{eqnarray*}
\sum_{t=1}^{T}\eta_{t}&\geq& \frac{\eta_0}{T\ln T}\sum_{t=1}^T(T-t) 
\end{eqnarray*}
Now, we use the fact that $\sum_{t=1}^TT-t=\sum_{t=1}^Tt$ and conclude that:
\begin{eqnarray*}
\sum_{t=1}^{T}\eta_{t}\geq \frac{\eta_0}{T\ln T}\sum_{t=1}^Tt=\frac{\eta_0(T+1)}{2\ln T}
,
\end{eqnarray*}
where the last equality is obtained from the fact that $\sum_{t=1}^Tt=\frac{T(T+1)}{2}$.}
\end{proof}
\begin{proof}{\bf Proof Lemma \ref{uper_bound}:}
To establish the lemma, we employ the fact that for the function 
$f(x)$, we have:
\begin{eqnarray*}
\sum_{i=1}^T f(x_i) \leq 1 + \int_1^T f(x)dx.
\end{eqnarray*}
Consequently, for the newly proposed step size, we obtain:
\begin{eqnarray}\label{up1}
\eta_0^2 \sum_{t=1}^T \left(1 - \frac{\ln t}{\ln T}\right)^2 \leq \eta_0 \left[1 + \int_1^T \left(1 - \frac{\ln t}{\ln T}\right)^2  dt\right].
\end{eqnarray}
To obtain an upper bound for (\ref{up1}) we will utilize the following integrals:
\begin{eqnarray}
\int \ln x  ~dx &=& x \ln x - x + C \nonumber \\
\int \ln^2 x ~dx &=& x \ln^2 x - 2x \ln x + 2x + C, \label{integ}
\end{eqnarray}
where $C$ represents a constant.
\\
The right-hand side of inequality (\ref{up1}) can be expressed as:
\begin{eqnarray}
= \eta_0^2 \left[1 + \int_1^T \frac{\ln^2 T - 2\ln T \ln t + \ln^2 t}{\ln^2 T} dt\right].\label{sec}
\end{eqnarray}
Now, utilizing (\ref{integ}) and considering that 
$\ln 1=0$, we can represent (\ref{sec}) as:
\begin{eqnarray*}
&=& \frac{\eta_0^2}{\ln^2 T} \left[\ln^2 T + (T-1)\ln^2 T\right] 
- 2 \frac{\eta_0^2}{\ln^2 T} \left[\ln T (T\ln T - T + 1)\right]
+ \frac{\eta_0^2}{\ln^2 T} \left[T\ln^2 T - 2T\ln T + 2T - 2\right] \\
&=& \frac{\eta_0^2 [2T - 2\ln T - 2]}{\ln^2 T} 
\leq \eta_0^2 \frac{2T}{\ln^2 T}.
\end{eqnarray*}
It implies that:
\begin{eqnarray*}
    \sum_{i=1}^T\eta_t^2\leq  \eta_0^2 \frac{2T}{\ln^2 T}.
\end{eqnarray*}
\end{proof}
\begin{proof}{\bf{Proof of Lemma \ref{uper_grad}}:}
    We know that $x_{t+1}=x_t-\eta_t\hat{g}_t$, where $\Bbb{E}[\hat{g}_t]=\nabla f(x_t)$. Using the smoothness of $f$, we have:
\begin{eqnarray}
    f(x_{t+1})
    &\leq& f(x_t)+\langle\nabla f(x_t),x_{t+1}-x_{t}\rangle
  +\frac{L}{2}\|x_{t+1}-x_{t}\|^2\nonumber\\
    &\leq&f(x_t)+\langle\nabla f(x_t),-\eta_t\hat{g}_t\rangle
   +
    +\frac{L}{2}\|x_{t+1}-x_{t}\|^2\nonumber\\
        &\leq&f(x_t)-\eta_t\langle\nabla f(x_t),\hat{g}_t\rangle+\frac{L\eta^2_t}{2}\|\hat{g}_t\|^2\label{smoo}
\end{eqnarray}
The stochastic gradient oracle is variance-bounded by $\sigma^2$, that is $\Bbb{E}[\|\hat{g}_t -\nabla f(x_t)\|^2]\leq \sigma^2$. Taking expectation on both sides of (\ref{smoo}) and applying $\Bbb{E}[\hat{g}_t] = \nabla f(x_t)$ and the variance-bounded assumption gives
\begin{eqnarray}
\Bbb{E}[f(x_{t+1})]
 &  \leq&\Bbb{E}[f(x_t)]-\eta_t\Bbb{E}[\|\nabla f(x_t)\|^2]+\frac{L\eta^2_t}{2}\Bbb{E}[\|\hat{g}_t\|^2]\nonumber\\
&    \leq&\Bbb{E}[f(x_t)]-\eta_t\Bbb{E}[\|\nabla f(x_t)\|^2]\nonumber+\frac{L\eta^2_t}{2}\Bbb{E}[\|\hat{g}_t
-\nabla f(x_t)+\nabla f(x_t)\|^2]\nonumber
\end{eqnarray}
Using the fact that gradient $\hat{g}_t$ is unbiased, i.e., $ \Bbb{E}[\hat{g}_t]=\nabla f(x_t)$. Therefore, we have:
  \begin{eqnarray}
\Bbb{E}[\|\hat{g}_t-\nabla f(x_t)\|^2]
&=&\Bbb{E}[\|\hat{g}_t-\nabla f(x_t)\|^2]
+\|\nabla f(x_t)\|^2
+2\langle \hat{g}_t-\nabla f(x_t),\nabla f(x_t)\rangle\nonumber\\
&=&\Bbb{E}[\|\hat{g}_t-\nabla f(x_t)\|^2]+\Bbb{E}[\|\nabla f(x_t)\|^2].\nonumber
\end{eqnarray}
Therefore, we can conclude that
 \begin{eqnarray}
   \Bbb{E}[f(x_{t+1})] &\leq& \Bbb{E}[f(x_t)]-\eta_t\Bbb{E}[\|\nabla f(x_t)\|^2]\frac{L\eta^2_t}{2}\left(\Bbb{E}[\|\hat{g}_t-\nabla f(x_t)\|^2]+\Bbb{E}[\|\nabla f(x_t)\|^2]\right)\nonumber\\
       &\leq&\Bbb{E}[f(x_t)]+\left(-\eta_t+\frac{L\eta^2_t}{2}\right)\Bbb{E}[\|\nabla f(x_t)\|^2]+\frac{L\eta^2_t}{2}\Bbb{E}[\|\hat{g}_t-\nabla f(x_t)\|^2]\nonumber\\
        &\leq&\Bbb{E}[f(x_t)]+\left(-\eta_t+\frac{L\eta^2_t}{2}\right)\Bbb{E}[\|\nabla f(x_t)\|^2]+\frac{L\eta^2_t\sigma^2}{2}.\nonumber
\end{eqnarray}
Using this fact that $\eta_t\leq\frac{1}{cL}\leq \frac{1}{L}$, we have $-\eta_t+\frac{L\eta^2_t}{2}\leq -\frac{\eta_t}{2}$. It implies that:
\begin{eqnarray}
    \frac{\eta_t}{2}\Bbb{E}[\|\nabla f(x_t)\|^2]&\leq& \Bbb{E}[f(x_t)]-\Bbb{E}[f(x_{t+1})]+\frac{L\eta^2_t\sigma^2}{2}\nonumber
\end{eqnarray}
\end{proof}
\begin{proof}{\bf{Proof of Corollary \ref{warm}}}\\
{Corollary \ref{comp} is true for $i=1,2,...,l$. Therefore,  we have:
    \begin{eqnarray*}
\min_i (\Bbb{E}\|\nabla f(\bar{x}_{T})\|^2)&\leq& \sum_{i=1}^l  \Bbb{E}\|\nabla f(\bar{x}_{T_i})\|^2 \\
  &\leq& \sum_{i=1}^{l}\left(\frac{4cL\ln T}{(T)}\left(f(x_{1_i})-f^*\right)+\frac{4\sigma^2T}{Lc(T+1)\ln T}\right)\\
   &\leq& l \max_i\left(\frac{4cL\ln T}{(T+1)}\left(f(x_{1_i})-f^*\right)+\frac{4\sigma^2T}{Lc(T+1)\ln T}\right)\\
   &=&\frac{4lcL\ln T}{T}\left(f(\tilde{x}_{1})-f^*\right)+\frac{4\sigma^2l}{Lc\ln T},
    \end{eqnarray*}
    in which $f(\tilde{x}_1)=\max_i\{f(x_{1_i})\}$.}
\end{proof}

\bibliographystyle{unsrtnat}
\bibliography{references}  






\end{document}